\documentclass[11pt,a4paper]{article}
\usepackage[dvipsnames]{xcolor}
\usepackage[hyperref]{acl2019}
\usepackage{times}
\usepackage{latexsym}
\usepackage{todonotes}
\usepackage{booktabs}

\usepackage{etoolbox}
\usepackage{amsmath,amssymb}
\usepackage{amsthm}
\usepackage{inconsolata}
\usepackage{bm}
\usepackage{nicefrac}
\usepackage{mathtools}
\usepackage{empheq}
\usepackage{url}
\usepackage{algorithm}
\usepackage{microtype}
\usepackage{xspace}
\usepackage{stmaryrd}

\usepackage{setspace}
\usepackage{algorithmicx}
\usepackage[noend]{algpseudocode}
\algrenewcommand\alglinenumber[1]{{\textcolor{gray}{\sf\scriptsize#1}}}
\algrenewcommand\algorithmicindent{1.3em}%

\usepackage{enumitem}
\usepackage{soul}
\definecolor{bananamania}{rgb}{0.98, 0.91, 0.71}
\sethlcolor{bananamania}

\definecolor{vspcol}{rgb}{0.78, 0.71, 0.61}
\newcommand\vsp{\textcolor{vspcol}{\,\textvisiblespace\,}}

\usepackage{inconsolata}
\usepackage{tikz}

\newtoggle{comms}

\togglefalse{comms}

\newcommand{\std}[1]{\footnotesize\textcolor{gray}{#1}}

\newcommand{\amap}{\bm{\pi}}
\newcommand{\reals}{\mathbb{R}}
\newcommand{\pfrac}[2]{\frac{\partial #1}{\partial #2}}
\newcommand{\simplex}{\triangle}
\newcommand\pp{p}
\newcommand\p{\bm{\pp}}
\newcommand\xx{z}
\newcommand\x{\bm{\xx}}
\newcommand{\eqnref}[1]{Eq.~\ref{eq:#1}}
\newcommand{\HHs}{\mathsf{H}^\textsc{S}}
\newcommand{\HHg}{\mathsf{H}^\textsc{G}}
\newcommand{\HHt}{\mathsf{H}^{\textsc{T}}}
\newcommand{\ones}{\bm{1}}
\newcommand{\EE}{\mathbb{E}}
\newcommand{\Loss}{\mathsf{L}}

\newcommand{\secref}[1]{\S\ref{sec:#1}}
\newcommand{\appref}[1]{Appendix~\ref{sec:#1}}

\newcommand{\DP}[2]{{#1}^\top{#2}}

\newcommand\thresh{\tau}

\DeclareMathOperator*{\argmax}{\mathsf{argmax}}
\DeclareMathOperator*{\argmin}{\mathsf{argmin}}
\DeclareMathOperator*{\diag}{\mathsf{diag}}

\DeclareMathOperator*{\softmax}{\mathsf{softmax}}
\DeclareMathOperator*{\sparsemax}{\mathsf{sparsemax}}
\newcommand*\entmaxtext{entmax\xspace}

\newcommand*\aentmax[1][\alpha]{\mathop{\mathsf{#1}\textnormal{-}\mathsf{\entmaxtext}}}

\let\tanh\relax
\let\log\relax
\let\exp\relax
\let\max\relax
\let\min\relax
\let\lim\relax
\DeclareMathOperator*{\tanh}{\mathsf{tanh}}
\DeclareMathOperator*{\log}{\mathsf{log}}
\DeclareMathOperator*{\exp}{\mathsf{exp}}
\DeclareMathOperator*{\max}{\mathsf{max}}
\DeclareMathOperator*{\min}{\mathsf{min}}
\DeclareMathOperator*{\lim}{\mathsf{lim}}

\newcommand*{\eg}{\textit{e.\hspace{.07em}g.}\@\xspace}
\newcommand*{\ie}{\textit{i.\hspace{.07em}e.}\@\xspace}
\newcommand*{\cf}{\textit{cf.}\@\xspace}

\newcommand{\langp}[2]{\textsc{#1}$\shortrightarrow$\textsc{#2}}
\newcommand{\langpb}[2]{\textsc{#1}$\leftrightarrow$\textsc{#2}}

\newtheorem{definition}{Definition}
\newtheorem{proposition}{Proposition}
\newtheorem{lemma}{Lemma}
\newtheorem{corollary}{Corollary}[lemma]

\aclfinalcopy %

\newcommand\sts{\emph{seq2seq}}

\makeatletter
\def\paragraph{\@startsection{paragraph}{4}{\z@}{1ex plus
   0.5ex minus .1ex}{-1em}{\normalsize\bf}}
\makeatother

\title{Sparse Sequence-to-Sequence Models%
}

\author{Ben Peters\textsuperscript{$\dag$} \quad
        Vlad Niculae\textsuperscript{$\dag$} \and
        Andr\'e F.~T. Martins\textsuperscript{$\dag\ddag$} \\
\textsuperscript{$\dag$}Instituto de Telecomunica\c{c}\~oes, Lisbon, Portugal \\
\textsuperscript{$\ddag$}Unbabel, Lisbon, Portugal\\
\href{mailto:benzurdopeters@gmail.com}{\tt benzurdopeters@gmail.com},\quad
\href{mailto:vlad@vene.ro}{\tt vlad@vene.ro},\quad
\href{mailto:andre.martins@unbabel.com}{\tt andre.martins@unbabel.com}
}

\date{}

\begin{document}
\maketitle
\begin{abstract}
Sequence-to-sequence models are a powerful workhorse of NLP. Most variants employ a softmax transformation in both their attention mechanism and output layer, leading to dense alignments and strictly positive output probabilities. This density is wasteful, making models less interpretable and assigning probability mass to many implausible outputs.
In this paper, we propose \emph{sparse} sequence-to-sequence models,
rooted in a new family of $\alpha$-\entmaxtext transformations, which includes softmax and sparsemax as particular cases, and is sparse for any $\alpha > 1$. We provide fast algorithms to evaluate these transformations and their gradients, which scale well for large vocabulary sizes.
Our models are able to produce sparse alignments and to assign nonzero probability to a short list of plausible outputs, sometimes rendering beam search exact.
Experiments on morphological inflection and machine translation reveal consistent gains over dense models.
\end{abstract}

\section{Introduction}

Attention-based sequence-to-sequence (\sts) models have proven useful for a
variety of NLP applications, including machine translation
\cite{bahdanau2014neural,vaswani2017attention},
speech recognition \cite{chorowski2015attention}, abstractive summarization
\cite{chopra2016abstractive}, and morphological inflection generation
\cite{kann2016med}, among others.  In part, their strength comes from their
flexibility: many tasks can be formulated as transducing a source sequence into
a target sequence of possibly different length.

However, conventional {\sts} models are {\bf dense}: they compute both attention weights and output probabilities with the {\bf softmax function} \citep{bridle1990probabilistic}, which always returns positive values. This results in \emph{dense attention alignments}, in which each source position is attended to at each target position, and in \emph{dense output probabilities}, in which each vocabulary type always has nonzero probability of being generated.
This contrasts with traditional statistical machine translation systems, which are based on sparse, hard alignments, and decode by navigating through a sparse lattice of phrase hypotheses.
Can we transfer such notions of sparsity to modern neural architectures? And if
so, do they improve performance?

In this paper, we provide an affirmative answer to both questions by proposing
{\bf neural sparse {\sts} models} that replace the softmax  transformations
(both in the attention and output) by {\bf sparse transformations}. Our
innovations are rooted in the recently proposed sparsemax transformation
\citep{martins2016softmax} and Fenchel-Young losses \citep{blondel2019learning}.
Concretely, we consider a family of transformations (dubbed {\bf \boldmath $\alpha$-\entmaxtext}),
parametrized by a scalar $\alpha$, based on the Tsallis entropies
\citep{Tsallis1988}. This family includes softmax ($\alpha=1$) and
sparsemax ($\alpha=2$) as particular cases. Crucially,
\entmaxtext transforms are sparse for all $\alpha>1$.

\begin{figure}

\definecolor{hlcolor}{rgb}{0.9, 0.17, 0.31}
\tikzset{greyline/.style={black!30,->}}
\tikzset{tok/.style={align=left}}
\tikzset{highlightline/.style={thick,hlcolor,->}}
\tikzset{perc/.style={hlcolor,inner sep=0, outer sep=0,align=right,anchor=west}}
\centering
\begin{tikzpicture}[node distance=1.5ex,font=\ttfamily]
\node (a1) at (0,0) {d};
\node[tok, right=of a1] (a2) {r};
\node[tok, right=of a2] (a3) {a};
\node[tok, right=of a3] (a4) {w};
\node[tok, right=of a4] (a5) {e};
\node[tok, right=of a5] (a6) {d};
\node[tok, right=of a6] (a7) {</s>};

\node[tok, below=of a5] (b5) {n};
\node[tok, right=of b5] (b6) {</s>};

\node[tok, below=of b5.south west, anchor=north west] (c5) {</s>};

\draw[greyline] (a1) -- (a2);
\draw[greyline] (a2) -- (a3);
\draw[greyline] (a3) -- (a4);
\draw[greyline] (a5) -- (a6);
\draw[greyline] (a6) -- (a7);
\draw[greyline] (b5) -- (b6);

\draw[highlightline] (a4) -- (a5);
\path (a4) edge[highlightline, bend right] (b5.west);
\path (a4) edge[highlightline, bend right] (c5.west);

\node[perc,above=.0cm of a5.north west]
    {\scriptsize 66.4\%};

\node[perc,above=.0cm of b5.north west]
    {\scriptsize 32.2\%};

\node[perc,above=.0cm of c5.north west]
    {\scriptsize 1.4\%};

\end{tikzpicture}
\caption{The full beam search of our best performing morphological inflection model
when generating the past participle of the verb ``draw''. The model gives nonzero probability
to exactly three hypotheses, including the correct form
(``drawn'') and the form that would be correct if ``draw'' were regular (``drawed'').\label{figure:draw}}
\end{figure}
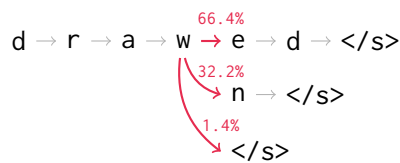

\begin{figure*}
\centering
\includegraphics[scale=1.0]{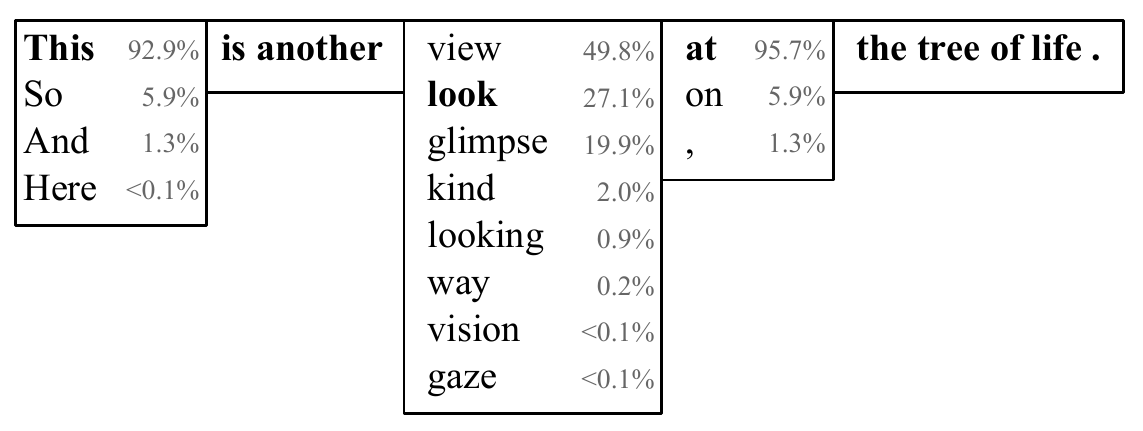}
\caption{Forced decoding using sparsemax attention and
1.5-\entmaxtext output for the German source sentence, ``Dies ist ein weiterer Blick auf
den Baum des Lebens.'' Predictions with nonzero probability are shown at each
time step. All other target types have probability exactly zero.
When consecutive predictions consist of a single word, we combine their borders
to showcase \emph{auto-completion} potential.
The selected gold targets are in boldface.\label{fig:tree}}
\end{figure*}

Our models are able to produce both {\bf sparse attention}, a form of inductive bias that increases focus on relevant source words and makes alignments more interpretable, and
{\bf sparse output probabilities}, which together with auto-regressive models can lead to probability distributions that are nonzero only for a finite subset of all possible strings.
In certain cases, a short list of plausible outputs can be enumerated without ever exhausting the beam
(Figure~\ref{figure:draw}), rendering beam search {\bf exact}. Sparse output
\sts{} models can also be used
for adaptive, sparse next word suggestion (Figure~\ref{fig:tree}).

Overall, our contributions are as follows:
\begin{itemize}
\item We propose an {\bf \entmaxtext sparse output layer}, together with a
natural loss function. In large-vocabulary settings, sparse outputs avoid wasting
probability mass on unlikely outputs, substantially improving accuracy.  For
tasks with little output ambiguity, \entmaxtext losses, coupled with beam
search, can often produce \textbf{exact finite sets} with only one or a few
sequences.
To our knowledge, this is the first study of sparse output probabilities
in {\sts} problems.

\item We construct \textbf{\entmaxtext sparse attention},
improving interpretability at no cost in accuracy.
We show that the \entmaxtext gradient has a simple form
(Proposition~\ref{prop:general_backward}),
revealing an insightful missing link between softmax and sparsemax.
\item We derive a novel exact algorithm for the case of 1.5-\entmaxtext,
achieving processing speed close to softmax on the GPU, even with large
vocabulary sizes. For arbitrary $\alpha$, we investigate a GPU-friendly
approximate algorithm.%
\footnote{Our standalone Pytorch entmax implementation is available at \url{https://github.com/deep-spin/entmax}.} %
\end{itemize}
We experiment on two tasks:
one character-level with little ambiguity ({\bf morphological inflection generation})
and another word-level, with more ambiguity ({\bf neural machine translation}).
The results show clear benefits of our approach, both in terms of accuracy and interpretability.

\section{Background}
\label{sec:background}
The underlying architecture we focus on is an RNN-based {\sts} with
global attention and input-feeding  \citep{luong2015effective}.
We provide a brief description of this architecture, with an emphasis on the
{attention mapping} and the {loss function}.

\paragraph{Notation.} Scalars, vectors, and matrices
are denoted respectively
as $a$, $\bm{a}$, and $\bm{A}$.
We denote the $d$--probability simplex (the set of vectors
representing probability distributions over $d$ choices) by
$\simplex^d \coloneqq \{\p \in \reals^d \colon \p \geq 0, \|\p\|_1 = 1\}$.
We denote the positive part as $[a]_+ \coloneqq \max\{a, 0\}$, and
by $[\bm{a}]_+$ its elementwise application to vectors.
We denote the indicator vector $\bm{e}_y \coloneqq
[0, \dots, 0, \underbrace{1}_{y}, 0, \dots, 0]$.

\paragraph{Encoder.} Given an input sequence of tokens $\bm{x} \coloneqq [x_1, \dots, x_J]$,
the encoder applies an embedding lookup followed by $K$ layered bidirectional LSTMs
\citep{hochreiter1997long}, resulting in encoder states $[ \bm{h}_1, \dots, \bm{h}_J]$.

\paragraph{Decoder.} The decoder generates output tokens $y_1, \ldots, y_T$, one at a time, terminated by a stop symbol. At each time step $t$,
it computes a probability distribution for the next generated word $y_t$, as follows.
Given the current state $\bm{s}_t$ of the decoder LSTM,  an {\bf attention mechanism} \citep{bahdanau2014neural} computes a focused,
fixed-size summary of the encodings $[ \bm{h}_1, \dots, \bm{h}_J]$, using $\bm{s}_t$ as a query vector. This is done by computing token-level scores $z_j \coloneqq \bm{s}_{t}^\top \bm{W}^{(z)}\bm{h}_{j}$,
then taking a weighted average
\begin{equation}\label{eq:attn}
{\bm{c}}_t \coloneqq \sum_{j=1}^{J} \pi_j \bm{h}_j, ~\text{where}~
\bm{\pi} \coloneqq \softmax(\bm{z}).
\end{equation}
The contextual output is the non-linear combination
$\bm{o}_t \coloneqq \tanh(\bm{W}^{(o)} [\bm{s}_t; \bm{c}_t] + \bm{b}^{(o)})$,
yielding the predictive distribution of the next word%
\newcommand{\tightdots}{\makebox[.8em][c]{.\hfil.\hfil.}}
\begin{equation}\label{eq:seqproba}%
p(y_t\!=\!\cdot \mid \bm{x},  y_1,\text{\tightdots}, y_{t-1})%
\coloneqq%
\softmax(\bm{V}\!\bm{o}_t + \bm{b}).
\end{equation}
The output $\bm{o}_t$, together with the embedding of the predicted
$\widehat{y}_{t}$, feed into the decoder LSTM for the next step, in an auto-regressive manner.
The model is trained to maximize the
likelihood of the correct target sentences, or equivalently, to minimize
\begin{equation}\label{eq:seqloss}
\mathcal{L} = \sum_{(x, y) \in \mathcal{D}}\,\sum_{t=1}^{|y|}
\underbrace{\left(-\log
\softmax(\bm{V}\!\bm{o}_t)\right)_{y_t}}_{\Loss_{\softmax}(y_t, \bm{V}\!\bm{o}_t)}.
\end{equation}
A central building block in the architecture is the transformation $\softmax
\colon \reals^d \rightarrow \simplex^d$,
\begin{equation}\label{eq:softmax}
\softmax(\bm{z})_j\!\coloneqq\!\frac{\exp (z_j)}{\sum_i \exp(z_i)},
\end{equation}
which maps a vector of scores $\bm{z}$ into a probability distribution (\ie, a
vector in $\simplex^d$).
As seen above, the $\softmax$ mapping plays \textbf{two crucial roles} in the decoder:
first, in computing normalized attention weights (\eqnref{attn}), second, in computing the predictive
probability distribution (\eqnref{seqproba}).
Since $\exp \gneq 0$, softmax never assigns a probability of zero to any
word, so we may never fully rule out non-important input tokens from attention,
nor unlikely words from the generation vocabulary. While this may be advantageous for
dealing with uncertainty, it may be preferrable to avoid dedicating model
resources to irrelevant words.
In the next section, we present a strategy for {\bf differentiable sparse
probability mappings}. We show that our approach can be used to learn
powerful {\sts} models with sparse outputs and sparse attention
mechanisms.

\section{Sparse Attention and Outputs}
\label{section:sparse-func}

\subsection{The sparsemax mapping and loss}

To pave the way to a more general family of sparse attention and losses,
we point out that softmax (\eqnref{softmax}) is only one of many possible
mappings from $\mathbb{R}^d$ to $\simplex^d$.  \citet{martins2016softmax}
introduce {\bf sparsemax}: an alternative to softmax which tends to yield
\textbf{sparse probability distributions}:
\begin{equation}\label{eq:sparsemax}
\sparsemax(\bm{z}) \coloneqq \argmin_{\p \in \simplex^{d}} \|\p - \x\|^2.
\end{equation}
Since \eqnref{sparsemax}  is a projection onto $\simplex^d$, which tends to
yield sparse solutions, the predictive distribution $\bm{p}^\star \coloneqq
\sparsemax(\bm{z})$ is likely to assign \textbf{exactly zero} probability to
low-scoring choices.  They also propose a corresponding \textbf{loss function}
to replace the negative log likelihood loss $\Loss_{\softmax}$
(\eqnref{seqloss}):
\begin{equation}\label{eq:sparsemaxloss}
\Loss_{\sparsemax}(y, \bm{z})\!\coloneqq\!%
\frac{1}{2}\!\left(\|\bm{e}_y\!-\!\bm{z}\|^2%
\!-\!%
\|\bm{p}^\star\!-\!\bm{z} \|^2\right)\!,\!\!%
\end{equation}
This loss is smooth and convex on $\bm{z}$ and has a {\bf margin}: it is zero if
and only if $\bm{z}_y \ge \bm{z}_{y'} + 1$ for any $y' \ne y$
\citep[Proposition~3]{martins2016softmax}.  Training models with the sparsemax
loss requires its gradient (\cf \appref{fy_losses}): \[
\nabla_{\bm{z}} \Loss_{\sparsemax}(y, \bm{z}) = -\bm{e}_y + \bm{p}^\star. \] For
using the sparsemax \emph{mapping} in an attention mechanism,
\citet{martins2016softmax} show that it is differentiable almost everywhere,
with
\[ \pfrac{\sparsemax(\bm{z})}{\bm{z}} = \diag(\bm{s}) - \frac{1}{\|\bm{s}\|_1}
        \bm{ss}^\top,
\]
where $s_j = 1$ if $p^\star_j > 0$, otherwise $s_j = 0$.

\paragraph{Entropy interpretation.} At first glance, sparsemax appears very
different from softmax, and a strategy for producing other sparse probability
mappings is not obvious.  However, the connection becomes clear when
considering the variational form of softmax \citep{wainwright_2008}:
\begin{equation}
    \softmax(\bm{z}) = \argmax_{\bm{p} \in \simplex^d} \bm{p}^\top\bm{z} +
    \HHs(\bm{p}),
\end{equation}
where $\HHs(\bm{\p}) \coloneqq -\sum_j p_j \log p_j$ is the well-known
Gibbs-Boltzmann-Shannon entropy with base $e$.

Likewise, letting $\HHg(\bm{p}) \coloneqq \frac{1}{2}\sum_j p_j(1 - p_j)$ be the {\bf Gini entropy},
we can rearrange \eqnref{sparsemax} as%
\begin{equation}%
\sparsemax(\bm{z}) = \argmax_{\bm{p} \in \simplex^d} \bm{p}^\top\bm{z} +
\HHg(\bm{p}),
\end{equation}
crystallizing the connection between softmax and sparsemax: they only differ in
the choice of {\bf entropic regularizer}.

\subsection{A new \entmaxtext mapping and loss family}

\begin{figure}[t]
\input{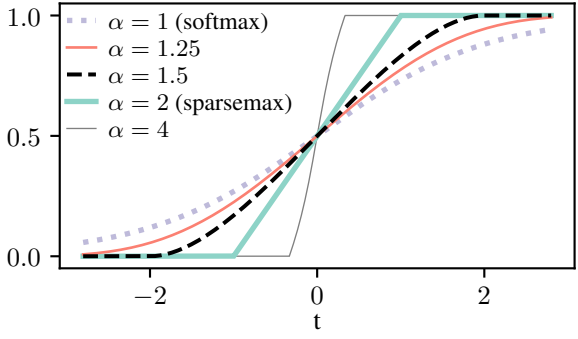}
\caption{\label{fig:entmax_mappings} Illustration of
\entmaxtext in the two-dimensional case $\aentmax([t, 0])_1$.
All mappings except softmax saturate at $t=\pm\nicefrac{1}{\alpha-1}$.
While sparsemax is piecewise linear, mappings with $1 < \alpha < 2$
have smooth corners.}
\end{figure}

The parallel above raises a question: can we find \textbf{interesting interpolations between
softmax and sparsemax?}  We answer affirmatively, by considering a generalization of the
Shannon and Gini entropies proposed by \citet{Tsallis1988}: a family of
entropies parametrized by a scalar $\alpha>1$ which we call {\bf Tsallis $\alpha$-entropies}:
\begin{equation}\label{eq:tsallisdef}
    \HHt_{\alpha}(\bm{p})\!\coloneqq\!
\begin{cases}
\frac{1}{\alpha(\alpha-1)}\sum_j\!\left(p_j - p_j^\alpha\right)\!, &
\!\!\!\alpha \ne 1,\\
\HHs(\bm{p}), &
\!\!\!\alpha = 1.
\end{cases}
\end{equation}

This family is continuous, \ie, $\lim_{\alpha \rightarrow
1}\HHt_{\alpha}(\bm{p}) = \HHs(\bm{p})$ for any $\bm{p} \in \simplex^d$
(\cf \appref{tsallis}). Moreover, $\HHt_{2} \equiv \HHg$. Thus,
Tsallis entropies interpolate between the Shannon
and Gini entropies.  Starting from the Tsallis entropies, we construct a
probability mapping, which we dub {\bf \entmaxtext}:
\begin{equation}\label{eq:define_entmax}
    \aentmax(\bm{z}) \coloneqq
    \argmax_{\p \in \simplex^d} \bm{p}^\top\bm{z} + \HHt_{\alpha}(\bm{p}),
\end{equation}
and,
denoting $\bm{p}^\star \coloneqq \aentmax(\bm{z})$,
a loss function
\begin{equation}\label{eq:define_entmax_loss}
    \Loss_{\alpha}(y, \bm{z}) \coloneqq
    (\bm{p}^\star -\bm{e}_y)^\top\bm{z} + \HHt_{\alpha}(\bm{p}^\star)
\end{equation}
The motivation for this loss function resides in the fact that it is a {\bf
Fenchel-Young loss} \citep{blondel2019learning}, as we briefly explain in \appref{fy_losses}.
Then,
$\aentmax[1]\equiv\softmax$ and  $\aentmax[2]\equiv\sparsemax$.  Similarly,
$\Loss_{1}$ is the negative log likelihood, and $\Loss_2$ is the sparsemax loss.
For all $\alpha > 1$, \entmaxtext tends to produce \textbf{sparse
probability distributions}, yielding a function family continuously
interpolating between softmax and sparsemax, \cf Figure~\ref{fig:entmax_mappings}.
The gradient of the \entmaxtext loss is%
\begin{equation}\label{eq:entmax_loss_grad}%
\nabla_{\bm{z}} \Loss_{\alpha}(y, \bm{z}) = -\bm{e}_y + \bm{p}^\star.
\end{equation}

\noindent {\bf Tsallis \entmaxtext losses} have useful properties including
convexity, differentiability, and a hinge-like separation margin property:
the loss incurred becomes zero when the score of the correct class is separated
by the rest by a margin of $\nicefrac{1}{\alpha - 1}$. When separation
is achieved, $\bm{p}^\star = \bm{e}_y$ \citep{blondel2019learning}.
This allows \entmaxtext~{\sts} models to be \textbf{adaptive to the degree
of uncertainty} present: decoders may make fully confident predictions at ``easy''
time steps, while preserving \textbf{sparse uncertainty} when a few choices are
possible (as exemplified in Figure~\ref{fig:tree}).

\noindent {\bf Tsallis \entmaxtext probability mappings} have not, to our
knowledge, been used in attention mechanisms.  They inherit
the desirable sparsity of sparsemax, while exhibiting smoother, differentiable
curvature, whereas sparsemax is piecewise linear.

\subsection{Computing the \entmaxtext mapping}
Whether we want to use $\aentmax$ as an attention mapping, or
$\Loss_\alpha$ as a loss function, we must be able to efficiently compute
$\p^\star = \aentmax(\x)$, \ie, to solve the maximization in
\eqnref{define_entmax}. For $\alpha=1$, the
closed-form solution is given by \eqnref{softmax}.
For $\alpha>1$, given $\x$, we show that there is a unique threshold $\thresh$
such that
(\appref{thresholded_form}, Lemma~\ref{lemma:tsallis_reduction}):
\begin{equation}
\aentmax(\x)
= [(\alpha - 1){\x} - \thresh \ones]_+^{\nicefrac{1}{\alpha-1}},
\end{equation}
\ie, entries with
score $\xx_j \leq \nicefrac{\thresh}{\alpha-1}$ get zero probability.
For sparsemax ($\alpha=2$), the problem amounts to Euclidean projection
onto $\simplex^d$, for which two types of algorithms are well studied:
{
    \setlist{parsep=3pt,itemsep=1pt,leftmargin=13pt,itemindent=\parindent,labelsep=6pt}
\begin{enumerate}[label=\roman*.]
    \item exact, based on sorting \cite{Held1974,michelot},
    \item iterative, bisection-based \cite{liu_efficient_2009}.
\end{enumerate}
}
The bisection approach searches for the optimal threshold $\tau$
numerically. \citet{blondel2019learning} generalize this approach in
a way applicable to $\aentmax$. The resulting algorithm is
(\cf \appref{thresholded_form} for details):

\begin{algorithm}[H]
\caption{Compute $\aentmax$ by bisection.}
\small\setstretch{1.15}
\begin{algorithmic}[1]
    \State Define $\p(\thresh) \coloneqq [\x - \thresh]_+^{\nicefrac{1}{\alpha - 1}}$,
set $\x \leftarrow (\alpha - 1) \x$
\State $\thresh_{\min} \leftarrow \max(\x) - 1; \thresh_{\max} \leftarrow \max(\x) - d^{1 - \alpha}$
\For{$t \in 1, \dots, T$}

\State  $\thresh \leftarrow (\thresh_{\min} + \thresh_{\max}) / 2$
\State $Z \leftarrow \sum_j \pp_j(\thresh)$
\State \textbf{if} {$Z<1$}~\textbf{then}~$\thresh_{\max}\leftarrow\thresh$~%
\textbf{else}~$\thresh_{\min}\leftarrow\thresh$
\EndFor
\State \Return $\nicefrac{1}{Z}~\p(\thresh)$%
\label{line:return_normalize}
\end{algorithmic}
\label{algo:bisect}
\end{algorithm}
Algorithm~\ref{algo:bisect} works by iteratively narrowing the interval containing the exact solution by exactly half.
Line \ref{line:return_normalize} ensures that approximate solutions are valid
probability distributions, \ie, that $\p^\star\in\simplex^d$.

Although bisection is simple and effective, an exact sorting-based algorithm, like
for sparsemax, has the potential to be faster and more accurate. Moreover, as
pointed out by \citet{condat2016}, when exact solutions are required, it is
possible to construct inputs $\x$ for which bisection requires arbitrarily many
iterations.  To address these issues, we propose a
\textbf{novel, exact algorithm for 1.5-\entmaxtext}, halfway between softmax and
sparsemax.

\begin{algorithm}[H]%
\caption{\label{alg:tsallis_15}Compute $\aentmax[1.5](\x)$ exactly.}
\small\setstretch{1.15}
\begin{algorithmic}[1]
\State Sort $\x$, yielding  $\xx_{[d]} \le \dots \le \xx_{[1]}$;
set $\x \leftarrow \nicefrac{\x}{2}$

\For{$\rho \in 1,\dots,d$}
    \State $M(\rho) \leftarrow \nicefrac{1}{\rho}
    \sum_{j=1}^{\rho}\!\xx_{[j]}$
    \State $\phantom{M}\llap{$S$\,}(\rho)\leftarrow \sum_{j=1}^{\rho}\!\left(\xx_{[j]} - M(\rho)\right)^2$
    \State $\phantom{M}\llap{$\tau$\,}(\rho) \leftarrow M(\rho) -
    \sqrt{\nicefrac{1}{\rho}\left(1 - S(\rho)\right)}$
    \vspace{1ex}
    \If{$\xx_{[\rho+1]} \leq \tau(\rho) \leq \xx_{[\rho]}$}
    \State \Return $\p^\star = \left[\x - \tau\,\bm{1} \right]_+^2$
    \EndIf
\EndFor
\end{algorithmic}
\end{algorithm}

We give a full derivation in \appref{algo}. As written, Algorithm~\ref{alg:tsallis_15}
is $O(d \log d)$ because of the sort; however, in practice, when the solution
$\p^\star$ has no more than $k$ nonzeros, we do not need to fully sort $\x$,
just to find the $k$ largest values. Our experiments in \secref{mt} reveal that
a partial sorting approach can be very efficient and competitive with softmax on
the GPU, even for large $d$. Further speed-ups might be available following the
strategy of \citet{condat2016}, but our simple incremental method is very easy
to implement on the GPU using primitives available in popular libraries
\citep{pytorch}.

Our algorithm resembles the aforementioned sorting-based algorithm for
projecting onto the simplex \citep{michelot}. Both algorithms rely on
the optimality conditions implying an analytically-solvable equation in
$\tau$: for sparsemax ($\alpha=2$), this equation is linear, for $\alpha=1.5$
it is quadratic (Eq.~\ref{eq:quad} in \appref{algo}).
Thus, exact algorithms may not be available for general values of $\alpha$.

\subsection{Gradient of the \entmaxtext mapping}

The following result shows how to compute the backward pass through
$\aentmax$, a requirement when using $\aentmax$ as an attention mechanism.
\begin{proposition}\label{prop:gradient}
Let $\alpha \geq 1$. Assume we have computed $\bm{p}^\star = \aentmax(\bm{z})$,
and define the vector
\[
    s_i = \begin{cases}
        (\pp^\star_i)^{2-\alpha}, & \pp^\star_i > 0, \\
        0, & \text{otherwise.} \\
    \end{cases}
\]
Then,\quad$\displaystyle \pfrac{\aentmax(\bm{z})}{\bm{z}}= \diag(\bm{s}) -
    \frac{1}{\|\bm{s}\|_1}~\bm{ss}^\top.$
\end{proposition}

\begin{proof}
The result follows directly from the more general
Proposition~\ref{prop:general_backward}, which we state and prove in
\appref{general_backward}, noting that
$\left(\frac{t^\alpha - t}{\alpha(\alpha-1)}\right)'' = t^{\alpha-2}$.
\end{proof}

\bigskip

The gradient expression recovers the softmax and sparsemax Jacobians with
$\alpha=1$ and $\alpha=2$, respectively \citep[Eqs.~8 and 12]{martins2016softmax}, thereby
providing another relationship between the two mappings.
Perhaps more interestingly, Proposition~\ref{prop:gradient}
shows why \textbf{the sparsemax Jacobian depends only on the
support} and not on the actual values of $\p^\star$: the sparsemax Jacobian is
equal for $\p^\star = [.99, .01, 0]$ and $\p^\star = [.5, .5, 0]$. This is not
the case for $\aentmax$ with $\alpha \neq 2$, suggesting that the gradients
obtained with other values of $\alpha$ may be more informative.
Finally, we point out that the gradient of \entmaxtext \emph{losses}
involves the \entmaxtext \emph{mapping} (\eqnref{entmax_loss_grad}), and
therefore Proposition~\ref{prop:gradient} also gives the \emph{Hessian} of
the \entmaxtext loss.

\section{Experiments}

The previous section establishes the computational building blocks required to train
models with \entmaxtext sparse attention and loss functions.
We now put them to use for two important NLP tasks, morphological inflection and machine translation.
These two tasks highlight the characteristics of our innovations in different ways.
Morphological inflection is a {\bf character-level} task with mostly monotonic alignments, but the evaluation demands exactness: the predicted sequence must match the gold standard.
On the other hand, machine translation uses a {\bf word-level} vocabulary orders of magnitude larger and
forces a sparse output layer to confront more ambiguity: any sentence has several valid translations and it is not clear beforehand that \entmaxtext will manage this well.

Despite the differences between the tasks, we keep the architecture and
training procedure as similar as possible. We use two layers for encoder and
decoder LSTMs and apply dropout with probability 0.3. We train with Adam \citep{kingma2014adam}, with a
base learning rate of 0.001, halved whenever the loss increases on the
validation set.
We use a batch size of 64.
At test time, we select the model with the best validation
accuracy and decode with a beam size of 5.
We implemented all models with OpenNMT-py \citep{2017opennmt}.\footnote{Our experiment code is at \url{https://github.com/deep-spin/OpenNMT-entmax}.}

In our primary experiments, we use three $\alpha$ values for the attention and loss functions: $\alpha=1$ (softmax), $\alpha=1.5$ (to which our novel Algorithm~\ref{alg:tsallis_15} applies), and $\alpha=2$ (sparsemax). We also investigate the effect of tuning $\alpha$ with increased granularity.

\subsection{Morphological Inflection}
The goal of morphological inflection is to produce an inflected word form (such as ``drawn'') given a lemma (``draw'')
and a set of morphological tags (\{\texttt{verb}, \texttt{past}, \texttt{participle}\}). We use the data from task 1 of the CoNLL--SIGMORPHON 2018 shared task \citep{cotterell2018conll}.
shared task
\paragraph{Training.} We train models under two data settings:
\textit{high} (approximately 10,000 samples per language in 86 languages) and \textit{medium} (approximately 1,000 training samples per language in 102 languages).
We depart from previous work by using
\textbf{multilingual training}: each model is trained on the data from
all languages in its data setting.
This allows parameters to be shared between languages, eliminates the need to train language-specific models, and may provide benefits similar to other forms of data augmentation  \citep{bergmanis2017training}.
Each sample is presented as a pair: the source contains the lemma concatenated
to the morphological tags and a special language identification token \citep{johnson2017google,peters2017massively},
and the target contains the inflected form.
As an example, the source sequence for Figure~\ref{figure:draw} is
\hl{\texttt{english{\vsp}verb{\vsp}participle{\vsp}past{\vsp}d{\vsp}r{\vsp}a{\vsp}w}}.
Although the set of inflectional tags is not sequential, treating it as such is simple to implement and works well in practice \citep{kann2016med}.
All models use embedding and hidden state sizes of 300. We validate at the end of every epoch in the \textit{high} setting and only once every ten epochs in \textit{medium} because of its smaller size.

\begin{table}[t]
\small

\begin{center}
\begin{tabular}{llrrrr}
\toprule
 \multicolumn{2}{c}{$\alpha$}   &  \multicolumn{2}{c}{high} & \multicolumn{2}{c}{medium} \\
output & attention & (avg.) &  (ens.) &  (avg.) &  (ens.)       \\
\midrule
1 & 1 & 93.15 & 94.20 & 82.55 & 85.68 \\
  & 1.5 & 92.32 & 93.50 & 83.20 & 85.63 \\
  & 2 & 90.98 & 92.60 & 83.13 & 85.65 \\
1.5 & 1 & 94.36 & 94.96 & 84.88 & 86.38 \\
  & 1.5 & 94.44 & 95.00 & 84.93 & 86.55 \\
  & 2 & 94.05 & 94.74 & 84.93 & 86.59 \\
2 & 1 & \textbf{94.59} & \textbf{95.10} & 84.95 & 86.41 \\
  & 1.5 & 94.47 & 95.01 & \textbf{85.03} & \textbf{86.61} \\
  & 2 & 94.32 & 94.89 & 84.96 & 86.47 \\
\midrule
\multicolumn{3}{l}{UZH \citeyearpar{makarov2018uzh}}  & 96.00 & & 86.64  \\
\bottomrule
\end{tabular}

\end{center}
\caption{Average per-language accuracy on the test set (CoNLL--SIGMORPHON 2018
task 1) averaged or ensembled over three runs.\label{table:multi-results}}
\end{table}

\paragraph{Accuracy.} Results are shown in Table~\ref{table:multi-results}.
We report the official metric of the shared task, word accuracy averaged across languages.
In addition to the average results of three individual model runs, we use an ensemble of
those models, where we decode by averaging the raw probabilities at each time step.
Our best sparse loss models beat the softmax baseline by nearly a full percentage point with ensembling, and up to two and a half points in the medium setting without ensembling.
The choice of attention has a smaller impact.
In both data settings, our best model on the validation set outperforms all submissions from the 2018 shared
task except for UZH \citep{makarov2018uzh},
which uses a more involved imitation learning approach and larger ensembles.
In contrast, our only departure from standard \sts{} training
is the drop-in replacement of softmax by \entmaxtext.

\paragraph{Sparsity.} Besides their accuracy,
we observed that \entmaxtext models made very sparse predictions: the best configuration in Table~\ref{table:multi-results} concentrates all probability mass into a single predicted sequence in 81\% validation samples in the \textit{high} data setting, and 66\% in the more difficult \textit{medium} setting.
When the model \emph{does} give probability mass to more than one sequence, the
predictions reflect reasonable ambiguity, as shown in Figure~\ref{figure:draw}.
Besides enhancing interpretability, sparsity in the output also has attractive properties for beam
search decoding: when the beam covers all nonzero-probability hypotheses, we
have a \textbf{certificate} of globally optimal decoding, {\bf rendering beam search exact}.
This is the case on 87\% of validation set sequences in the \textit{high} setting, and 79\% in \textit{medium}.
To our knowledge, this is the first instance of a neural {\sts} model that can offer optimality guarantees.

\subsection{Machine Translation}\label{sec:mt}

\begin{table*}[t]

\begin{center}
\small
\begin{tabular}{lr@{~}lr@{~}lr@{~}lr@{~}lr@{~}lr@{~}l}
\toprule
method
& \multicolumn{2}{c}{\langp{de}{en}} & \multicolumn{2}{c}{\langp{en}{de}}
& \multicolumn{2}{c}{\langp{ja}{en}} & \multicolumn{2}{c}{\langp{en}{ja}}
& \multicolumn{2}{c}{\langp{ro}{en}} & \multicolumn{2}{c}{\langp{en}{ro}} \\
\midrule
$\softmax$ & 25.70 &  \std{$\pm$ 0.15}
& 21.86 & \std{$\pm$ 0.09}
& 20.22 & \std{$\pm$ 0.08}
& 25.21 & \std{$\pm$ 0.29}
& 29.12 & \std{$\pm$ 0.18}
& 28.12 & \std{$\pm$ 0.18} \\
$\aentmax[1.5]$
& \textbf{26.17} & \std{$\pm$ 0.13}
& \textbf{22.42} & \std{$\pm$ 0.08}
& \textbf{20.55} & \std{$\pm$ 0.30}
& \textbf{26.00} & \std{$\pm$ 0.31}
& \textbf{30.15} & \std{$\pm$ 0.06}
& \textbf{28.84} & \std{$\pm$ 0.10} \\
$\sparsemax$
& 24.69 & \std{$\pm$ 0.22}
& 20.82 & \std{$\pm$ 0.19}
& 18.54 & \std{$\pm$ 0.11}
& 23.84 & \std{$\pm$ 0.37}
& 29.20 & \std{$\pm$ 0.16}
& 28.03 & \std{$\pm$ 0.16} \\
\bottomrule
\end{tabular}
\end{center}
\caption{Machine translation comparison of
softmax, sparsemax, and the proposed 1.5-\entmaxtext
as both attention mapping and loss function. Reported is tokenized test BLEU averaged across three runs (higher is better).\label{table:mt}}
\end{table*}

We now turn to a highly different \sts{} regime in which the vocabulary size is much larger, there is a great deal of ambiguity, and sequences can generally be translated in several ways.
We train models for three language pairs in both directions:

\begin{itemize}[itemsep=.5ex,leftmargin=2ex]
    \item IWSLT 2017 German $\leftrightarrow$ English
    \citep[\langpb{de}{en},][]{cettolooverview}: training size 206,112.
    \item KFTT Japanese $\leftrightarrow$ English
    \citep[\langpb{ja}{en},][]{neubig11kftt}: training size of 329,882.
    \item WMT 2016 Romanian $\leftrightarrow$ English
    \citep[\langpb{ro}{en},][]{bojar2016findings}: training size
    612,422, diacritics removed \citep[following][]{sennrich2016edinburgh}.
\end{itemize}

\paragraph{Training.} We use byte pair encoding \citep[BPE;][]{sennrich2016neural} to ensure an open vocabulary.
We use separate segmentations with 25k merge operations per language for \langpb{ro}{en} and a joint segmentation with 32k merges for the other language pairs. \langpb{de}{en} is validated once every 5k steps because of its smaller size, while the other sets are validated once every 10k steps. We set the maximum number of training steps at 120k for \langpb{ro}{en} and 100k for other language pairs. We use 500 dimensions for word vectors and hidden states.

\paragraph{Evaluation.} Table~\ref{table:mt} shows BLEU scores \citep{papineni2002bleu} for the three models with $\alpha \in \{1, 1.5, 2\}$, using the same value of $\alpha$ for the attention mechanism and loss function.
We observe that the 1.5-\entmaxtext configuration consistently performs best across all six choices of language pair and direction.
These results support the notion that the optimal function is somewhere between softmax and
sparsemax, which motivates a more fine-grained search for $\alpha$; we explore this next.

\paragraph{Fine-grained impact of {\boldmath $\alpha$}.}
Algorithm~\ref{algo:bisect} allows us to further investigate the marginal effect
of varying the attention $\alpha$ and the loss $\alpha$, while keeping the other
fixed.  We report \langp{de}{en} validation accuracy on a fine-grained $\alpha$
grid in Figure~\ref{fig:alpha_attn}.  On this dataset, moving from softmax
toward sparser \textbf{attention} (left) has a very small positive effect on
accuracy, suggesting that the benefit in interpretability does not hurt
accuracy.  The impact of the \textbf{loss function} $\alpha$ (right) is much
more visible: there is a distinct optimal value  around $\alpha=1.33$, with
performance decreasing for too large values.  Interpolating between softmax and
sparsemax thus inherits the benefits of both, and our novel
Algorithm~\ref{alg:tsallis_15} for $\alpha=1.5$ is confirmed to strike a good
middle ground.  This experiment also confirms that bisection is effective in
practice, despite being inexact. Extrapolating beyond the
sparsemax loss ($\alpha>2$) does not seem to perform well.

\begin{figure*}[t]
\centering
\input{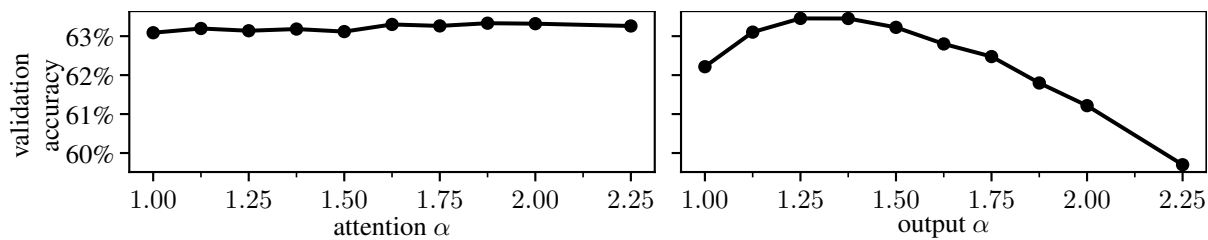}

\caption{Effect of tuning $\alpha$ on \langp{de}{en}, for
attention (left) and for output (right), while keeping the other $\alpha=1.5$.\label{fig:alpha_attn}}
\end{figure*}

\paragraph{Sparsity.} In order to form a clearer idea of how sparse \entmaxtext becomes, we measure the average number of nonzero indices on the \langp{de}{en} validation set and show it in Table~\ref{table:sparsity}.
As expected, 1.5-\entmaxtext is less sparse than sparsemax as both an attention mechanism and output layer.
In the attention mechanism, 1.5-\entmaxtext's increased support size does not come at the cost of much interpretability, as Figure~\ref{fig:attn-plot} demonstrates.
In the output layer, 1.5-\entmaxtext assigns positive probability to only 16.13 target types out of a vocabulary of 17,993, meaning that the supported set of words often has an intuitive interpretation.
Figure~\ref{fig:tree} shows the sparsity of the 1.5-\entmaxtext output layer in practice: the support becomes completely concentrated when generating a phrase like ``the tree of life'', but grows when presenting a list of synonyms (``view'', ``look'', ``glimpse'', and so on).
This has potential practical applications as a \textbf{predictive translation} system \citep{green2014human}, where the model's support set serves as a list of candidate auto-completions at each time step.

\begin{figure}
\centering
\includegraphics{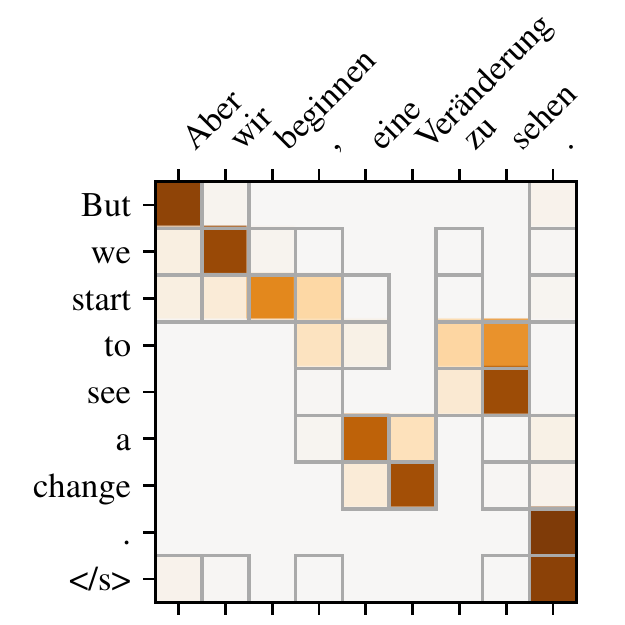}
\caption{Attention weights produced by the \langp{de}{en} {1.5-\entmaxtext} model. Nonzero weights are outlined.}
\label{fig:attn-plot}
\end{figure}

\begin{table}[t]

\begin{center}\small
\begin{tabular}{lrr}
\toprule
method  & \# attended &  \# target words \\
\midrule
$\softmax$ &     24.25 & 17993 \\
$\aentmax[1.5]$ &   5.55    &     16.13 \\
$\sparsemax$ &      3.75 &   7.55  \\
\bottomrule
\end{tabular}
\end{center}
\caption{Average number of nonzeros in the attention and output distributions for the \langp{de}{en} validation set.}
\label{table:sparsity}
\end{table}

\paragraph{Training time.}
Importantly, the benefits of sparsity do \textbf{not} come at a high computational cost. Our proposed
Algorithm~\ref{alg:tsallis_15} for 1.5-\entmaxtext runs
on the GPU at near-softmax speeds (Figure~\ref{fig:valid_curve}).
For other $\alpha$ values, bisection (Algorithm~\ref{algo:bisect}) is slightly
more costly, but practical even for large vocabulary sizes.
On \langp{de}{en}, bisection is capable of processing about 10,500 target words per second
on a single Nvidia GeForce GTX 1080 GPU, compared to 13,000 words per second for 1.5-\entmaxtext with Algorithm~\ref{alg:tsallis_15} and 14,500 words per second with softmax.
On the smaller-vocabulary morphology datasets, Algorithm~\ref{alg:tsallis_15} is nearly as fast as softmax.

\section{Related Work}

\paragraph{Sparse attention.}
Sparsity in the attention and in the output have different, but related, motivations.
Sparse attention can be justified as a form of inductive bias, since for tasks such as machine translation one expects only a few source words to be relevant for each translated word. Dense attention probabilities are particularly harmful for long sequences, as shown by \citet{luong2015effective}, who propose ``local attention'' to mitigate this problem. Combining sparse attention with fertility constraints has been recently proposed by \citet{malaviya2018sparse}.
Hard attention \citep{xu2015show,aharoni2017morphological,wu2018hard}
selects exactly one source token.
Its discrete, non-differentiable nature requires
imitation learning or Monte Carlo policy gradient approximations, which
drastically complicate training. In contrast, \entmaxtext is a differentiable,
easy to use, drop-in softmax replacement.
A recent study by \citet{Jain2019attention} tackles the limitations of attention
probabilities to provide interpretability. They only study dense
attention in classification tasks, where attention is less crucial for the
final predictions. In their conclusions, the authors defer to
future work exploring sparse attention mechanisms and {\sts} models. We believe
our paper can foster interesting investigation in this area.

\begin{figure}[t]%
\centering
\input{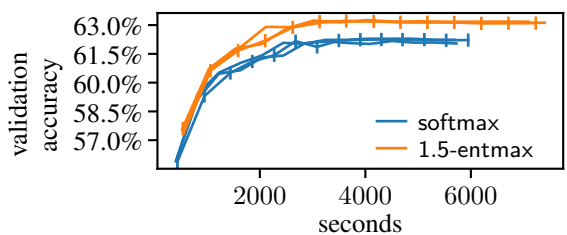}
\caption{Training timing on three \langp{de}{en} runs.
Markers show validation checkpoints for one of the runs.
\label{fig:valid_curve}}
\end{figure}

\paragraph{Losses for \sts{} models.}
Mostly motivated by the challenges of large vocabulary sizes in \sts{},
an important research direction tackles replacing the cross-entropy loss
with other losses or approximations \citep{bengio2008adaptive,morin2005hierarchical,kumar2018mises}.
While differently motivated, some of the above strategies (\eg, hierarchical
prediction) could be combined with our proposed sparse losses.
\citet{sparsemap} use sparsity to predict interpretable sets of structures.
Since auto-regressive \sts{} makes no factorization assumptions, their strategy
cannot be applied without approximations, such as in \citet{edunov2018classical}.

\section{Conclusion and Future Work}

We proposed sparse sequence-to-sequence models and provided fast algorithms to compute their attention and output transformations. Our approach yielded consistent improvements over dense models on morphological inflection and machine translation, while inducing interpretability
in both attention and output distributions. Sparse output layers also
provide exactness when the number of possible hypotheses does not exhaust beam search.

Given the ubiquity of softmax in NLP, \entmaxtext has many potential applications. A natural next step is to apply \entmaxtext to self-attention \citep{vaswani2017attention}.
In a different vein, the strong  morphological inflection results point to usefulness in other tasks where probability is concentrated in a small number of hypotheses, such as speech recognition.

\section*{Acknowledgments}
This work was %
supported by the European Research Council (ERC StG DeepSPIN 758969),
and by the Funda\c{c}\~ao para a Ci\^encia e Tecnologia 
through contracts UID/EEA/50008/2019 and CMUPERI/TIC/0046/2014 (GoLocal). 
We thank
Mathieu Blondel,
Nikolay Bogoychev,
Gon\c{c}alo Correia,
Erick Fonseca,
Pedro Martins,
Tsvetomila Mihaylova,
Miguel Rios,
Marcos Treviso,
and the anonymous reviewers,
for helpful discussion and feedback.

\bibliography{acl2019}
\bibliographystyle{acl_natbib}

\clearpage
\onecolumn
\appendix
\begin{center}
{\huge \textbf{Supplementary Material}}
\end{center}

\setlength{\parindent}{0pt}
\setlength{\parskip}{1.5ex plus 0.5ex minus .2ex}

\def\RR{{\mathbb{R}}}
\def\EE{{\mathbb{E}}}
\def\RRY{\RR^{|\cY|}}
\def\y{\bm{y}}
\def\triangleY{\triangle^{|\cY|}}
\def\sizeY{{|\cY|}}
\def\cC{{\mathcal{C}}}
\def\cD{{\mathcal{D}}}
\def\cX{{\mathcal{X}}}
\def\cY{{\mathcal{Y}}}

\section{Background}

\subsection{Tsallis entropies}\label{sec:tsallis}

Recall the definition of the Tsallis family of entropies in
\eqnref{tsallisdef} for $\alpha \geq 1$,
\begin{equation}
    \HHt_{\alpha}(\bm{p})\coloneqq
\begin{cases}
\frac{1}{\alpha(\alpha-1)}\sum_j\!\!\left(p_j - p_j^\alpha\right), &
\text{if } \alpha \ne 1,\\
\HHs(\bm{p}), &
\text{if } \alpha = 1.
\end{cases}
\end{equation}
This family is \textbf{continuous} in $\alpha$, \ie,
$\lim_{\alpha \rightarrow 1} \HHt_{\alpha}(\p) = \HHt_{1}(\p)$ for any
$\p \in \simplex^d$.
\begin{proof} For simplicity, we rewrite $\HHt_{\alpha}$ in separable form:
\begin{equation}
\HHt_\alpha(\p) = \sum_j h_{\alpha}(\pp_j)
\quad\text{with}\quad
h_{\alpha}(t) \coloneqq \begin{cases}
\frac{t - t^{\alpha}}{\alpha(\alpha-1)}, & \alpha > 1,\\
-t \log t, & \alpha = 1.
\end{cases}
\end{equation}
It suffices to show that
$\lim_{\alpha \rightarrow 1} h_{\alpha}(t) = h_{1}(t)$ for $t \in [0, 1]$.
Let $f(\alpha) \coloneqq t - t^\alpha$, and
$g(\alpha)\coloneqq\alpha(\alpha-1)$. Observe that
$\frac{f(1)}{g(1)}
= \nicefrac{0}{0}$,
so we are in an indeterminate case.
We take the derivatives of $f$ and $g$:%
\[%
f'(\alpha) = 0 - \left(\exp (\log t^\alpha) \right)'
= -\exp (\log t^\alpha) \cdot (\alpha \log t)'
= -t^\alpha \log t,
\quad\text{and}\quad
g'(\alpha) = 2\alpha - 1.
\]
From l'H\^{o}pital's rule,
\[
\lim_{\alpha \rightarrow 1} \frac{f(\alpha)}{g(\alpha)}
=\lim_{\alpha \rightarrow 1} \frac{f'(\alpha)}{g'(\alpha)}
= -t \log t
= h_1(t).
\]
\end{proof}
Note also that, as $\alpha \rightarrow \infty$, the denominator grows unbounded,
so $\HHt_\infty \equiv 0$.

\subsection{Fenchel-Young losses}\label{sec:fy_losses}

In this section, we recall the definitions and properties essential for our
construction of $\aentmax$. The concepts below were formalized by
\citet{blondel2019learning} in more generality; we present below a
less general version, sufficient for our needs.

\begin{definition}[Probabilistic prediction function regularized by $\Omega$]

Let $\Omega \colon \simplex^d \to \RR \cup \{\infty\}$ be a strictly convex
regularization function. We define the prediction function $\amap_{\Omega}$ as
\begin{equation}
\amap_{\Omega}(\x) \in  \argmax_{\p \in \simplex^d}
    \big(\p^\top\x - \Omega(\p)\big)
\label{eq:amap}
\end{equation}
\end{definition}
\begin{definition}[Fenchel-Young loss generated by $\Omega$]\label{def:FY_loss}

Let $\Omega \colon \simplex^d \to \RR \cup \{\infty\}$ be a strictly convex
regularization function. Let $\y \in \simplex$ denote a ground-truth label
(for example, $\y = \bm{e}_y$ if there is a unique correct class $y$).
Denote by $\x \in \RR^d$ the prediction scores produced by some model,
and by $\p^\star \coloneqq \amap_{\Omega}(\x)$ the probabilistic predictions.
The {\bf Fenchel-Young loss} $\Loss_\Omega \colon \RR^d \times \simplex \to \RR_+$
generated by $\Omega$ is
\begin{equation}
L_{\Omega}(\x; \y) \coloneqq
\Omega(\y) - \Omega(\p^\star) + \x^\top(\p^\star - \y).
\label{eq:fy_losses}
\end{equation}
\end{definition}

This justifies our choice of \entmaxtext mapping and loss
(Eqs.~\ref{eq:define_entmax}--\ref{eq:define_entmax_loss}), as
$\amap_{-\HHt_{\alpha}} = \aentmax$ and $\Loss_{-\HHt_{\alpha}} = \Loss_{\alpha}$.

\paragraph{Properties of Fenchel-Young losses.}
\begin{enumerate}[itemsep=3pt]

\item {\bf Non-negativity.} $L_{\Omega}(\x; \y) \ge 0$ for any
    $\x \in \RR^d$ and $\y \in \simplex^d$.

\item {\bf Zero loss.}
$L_\Omega(\x; \y) = 0$ if and only if $\y = \amap_{\Omega}(\x)$, \ie, the
prediction is exactly correct.

\item {\bf Convexity.} $\Loss_{\Omega}$ is convex in $\x$.

\item {\bf Differentiability.}
    $\Loss_{\Omega}$ is differentiable with $\nabla \Loss_{\Omega}(\x; \y) =
    \p^\star - \y$.

\item {\bf Smoothness.} If $\Omega$ is strongly convex, then $L_\Omega$ is
    smooth.

\item {\bf Temperature scaling.} For any constant $t > 0$,
    $\Loss_{t \Omega}(\x; \y) = t \Loss_{\Omega}(\nicefrac{\x}{t}; \y)$.
\end{enumerate}

\paragraph{Characterizing the solution $\p^\star$ of $\amap_{\Omega}(\x)$.}
To shed light on the generic probability mapping in \eqnref{amap},
we derive below the optimality conditions characterizing its solution. The
optimality conditions are essential not only for constructing algorithms for
computing $\p^\star$ (\appref{allalgo}), but also for deriving the Jacobian of the mapping
(\appref{general_backward}). The Lagrangian of the maximization in \eqnref{amap} is
\begin{equation}
    \mathcal{L}(\p, \bm{\nu}, \thresh) = \Omega(\p) - (\x + \bm{\nu})^\top{\p} + \thresh (\bm{1}^\top \p - 1).
\end{equation}
with subgradient
\begin{equation}
    \partial_{\p} \mathcal{L}(\p, \bm{\nu}, \thresh) =
\partial \Omega(\p) - \x - \bm{\nu} + \thresh \mathbf{1}.
\end{equation}
The subgradient KKT conditions are therefore:
\begin{empheq}[left=\empheqlbrace]{align}
    \x + \bm{\nu} - \thresh \mathbf{1} &\in \partial \Omega(\p)
    \label{eq:stationarity}\\
\DP{\p}{\bm{\nu}} &= 0
    \label{eq:complementary_slack}\\
\p &\in \triangle^d
    \label{eq:primal_feas}\\
\bm{\nu} &\ge 0.
    \label{eq:dual_feas}
\end{empheq}

\paragraph{Connection to softmax and sparsemax.} We may now directly see that,
when $\Omega(\p)\coloneqq \sum_j \pp_j \log \pp_j$, \eqnref{stationarity}
becomes $\log \pp_j = \x_j + \nu_j - \thresh - 1$, which can only be satisfied
if $\pp_j > 0$, thus $\bm{\nu}=\bm{0}$. Then, $\pp_j = \nicefrac{\exp(\xx_j)}{Z}$,
where $Z\coloneqq \exp (\thresh + 1)$.
From \eqnref{primal_feas}, $Z$ must be such that $\pp_j$ sums to 1,
yielding the well-known softmax expression. In the case of sparsemax, note that
for any $\p \in \simplex^d$, we have
\[\Omega(\p) = -\HHg(\p) = \nicefrac{1}{2} \sum_j \pp_j  (\pp_j - 1) =
\nicefrac{1}{2} \| \p \|^2 - \frac{1}{2} \underbrace{\sum_j \pp_j}_{=1} =
\nicefrac{1}{2} \| \p \|^2 + \text{const}.\]
Thus,
$\displaystyle
    \argmax_{\p \in \simplex^d} \p^\top\x + \HHg(\p) =
    \argmin_{\p \in \simplex^d} \textcolor{gray}{0.5\Big(} \|\p\|^2 - 2\p^\top\x
\textcolor{gray}{\left(+ \|\x\|^2\right)\Big)} =
    \argmin_{\p \in \simplex^d} \|\p - \x\|^2.
$

\section{Backward pass for generalized sparse attention mappings}%
\label{sec:general_backward}%
\newcommand{\jac}{\pfrac{{\amap_{\Omega}}}{{\x}}}

When a mapping $\amap_{\Omega}$ is used inside the computation graph of a neural
network, the Jacobian of the mapping has the important role of showing how
to propagate error information, necessary when training with gradient methods.
In this section, we derive a new, simple expression for the Jacobian of
generalized sparse $\amap_{\Omega}$. We apply this result to obtain a simple
form for the Jacobian of $\aentmax$ mappings.

The proof is in two steps. First, we prove a lemma that shows that
Jacobians are zero outside of the support of the solution. Then, completing the
result, we characterize the Jacobian at the nonzero indices.

\begin{lemma}[Sparse attention mechanisms have sparse Jacobians]\label{lemma:sparsejac}
Let $\Omega : \reals^d \rightarrow \reals$ be strongly convex.
The attention mapping $\amap_\Omega$
is differentiable almost everywhere, with Jacobian $\displaystyle \jac$
symmetric and satisfying
\[
    \pfrac{(\amap_{\Omega}(\x))_i}{\xx_j} = 0
    \quad\text{if}\quad(\amap_{\Omega}(\x))_i = 0
    \quad\text{or}\quad(\amap_{\Omega}(\x))_j = 0.
\]
\end{lemma}
\begin{proof}
Since $\Omega$ is strictly convex, the $\argmax$ in \eqnref{amap} is unique.
Using Danskin's theorem \citep{danskin_theorem}, we may write
\[
    \amap_{\Omega}(\x) = \nabla \max_{\p \in \simplex} \left(\p^\top\x - \Omega(\p)\right)
    = \nabla\Omega^*(\x).
\]
Since $\Omega$ is strongly convex, the gradient of its conjugate $\Omega^*$
is differentiable almost everywhere \citep{Rockafellar1970}. Moreover,
$\jac$ is the Hessian of $\Omega^*$, therefore it is symmetric, proving the
first two claims.
\\[\baselineskip]
Recall the definition of a partial derivative,
\[
    \pfrac{(\amap_{\Omega}(\x))_i}{\xx_j} = \lim_{\varepsilon \rightarrow 0}
    \frac{1}{\varepsilon} \left( \amap_{\Omega}(\x + \varepsilon \bm{e}_j)_i - \amap_{\Omega}(\x)_i \right).
\]
Denote by $\p^\star \coloneqq \amap_{\Omega}(\x)$. We will show that for any $j$ such
that $\pp^\star_j = 0$, and any $\varepsilon \geq 0$,
\[
\amap_{\Omega}(\x - \varepsilon \bm{e}_j) =
\amap_{\Omega}(\x) = \p^\star.
\]
In other words, we consider only one side of the limit, namely
\textbf{subtracting} a small non-negative $\varepsilon$.
A vector $\p^\star$ solves the optimization problem in \eqnref{amap} if and only
if there exists $\bm{\nu}^\star \in \reals^d$ and $\tau^\star \in \reals$
satisfying Eqs.~\ref{eq:stationarity}--\ref{eq:dual_feas}.
Let $\bm{\nu}_\varepsilon \coloneqq \bm{\nu}^\star + \varepsilon\bm{e}_j$.
We verify that $(\p^\star, \bm{\nu}_\varepsilon, \tau^\star)$
satisfies the optimality conditions for
$\amap_{\Omega}(\x - \varepsilon \bm{e}_j)$,
which implies that $\amap_{\Omega}(\x - \varepsilon \bm{e}_j) = \amap_{\Omega}(\x)$.
Since we add a non-negative
quantity to $\bm{\nu}^\star$, which is non-negative to begin with,
$(\bm{\nu}_\varepsilon)_j \geq 0$, and since $\pp^\star_j=0$,
we also satisfy $\pp^\star_j (\bm{\nu}_\varepsilon)_j=0$. Finally,
\[
    \x - \varepsilon \bm{e}_j + \bm{\nu}_\varepsilon - \tau^\star\bm{1}
    =
    \x - \varepsilon \bm{e}_j + \bm{\nu}^\star + \varepsilon \bm{e}_j - \tau^\star\bm{1}
    \in \partial \Omega(\p^\star).
\]
It follows that
$
\lim_{\varepsilon \rightarrow 0_{-}}
\frac{1}{\varepsilon} \left( \amap_{\Omega}(\x + \varepsilon \bm{e}_j)_i -
\amap_{\Omega}(\x)_i \right)
= 0.$ If $\amap_\Omega$ is differentiable at $\x$, this
one-sided limit must agree with the derivative. Otherwise,
the sparse one-sided limit is a generalized Jacobian.
\end{proof}
\begin{proposition}\label{prop:general_backward}
Let $\p^\star \coloneqq \amap_{\Omega}(\x)$,
with strongly convex and differentiable $\Omega$.
Denote the support of $\p^\star$ by $\mathcal{S} = \big\{j \in \{1, \dots, d\}:
\pp_j > 0\big\}$.
If the second derivative
$h_{ij} = \frac{\partial^2 \Omega}{\partial \pp_i \partial \pp_j}(\p^\star)$
exists for any $i,j \in \mathcal{S}$, then
\[
    \jac = \bm{S} - \frac{1}{\|\bm{s}\|_1}~\bm{ss}^\top
    \quad\text{where}\quad
    \bm{S}_{ij} = \begin{cases}
            \bm{H}^{-1}_{ij}, & i, j \in \mathcal{S}, \\
            0, & \text{o.w.} \\
    \end{cases},
    \quad\text{and}~
    \bm{s} = \bm{S}\bm{1}.
\]
In particular, if $\Omega(\p) = \sum_j g(\pp_j)$ with
$g$ twice differentiable on $(0, 1]$, we have
\[
    \jac = \diag{\bm{s}} - \frac{1}{\|\bm{s}\|_1}~\bm{ss}^\top
    \quad\text{where}\quad
    s_i = \begin{cases}
        \big(g''(\pp^\star_i)\big)^{-1}, & i \in \mathcal{S}, \\
        0, & \text{o.w.} \\
    \end{cases}
\]
\end{proposition}
\begin{proof}
Lemma~\ref{lemma:sparsejac}
verifies that
$\pfrac{(\amap_{\Omega})_i}{\xx_j} = 0$
for $i, j \notin \mathcal{S}$.
It remains to find the derivatives w.r.t.\ $i,j \in \mathcal{S}$.
Denote by $\bar{\p}^\star, \bar{\x}$ the restriction of the corresponding vectors to
the indices in the support $\mathcal{S}$.
The optimality conditions on the support are
\begin{eqnarray}
\left\{
\begin{array}{rl}
    \bm{g}(\bar{\p}) + \tau \bm{1} &= \bar{\x} \\
    \bm{1}^\top \bm{\bar{p}} &= 1
\end{array}
\right.
\end{eqnarray}
where $\bm{g}(\bar{\p})\coloneqq\big(\nabla \Omega(\p)\big)\big|_\mathcal{S}$,
so $\pfrac{\bm{g}}{\bar{\p}}(\bar{\p}^\star) = \bm{H}$.
Differentiating w.r.t.\ $\bar{\x}$ at $\p^\star$ yields
\begin{eqnarray}
\left\{
\begin{array}{rl}
    \bm{H} \pfrac{\bar{\p}}{\bar{\x}} +
    \bm{1}\pfrac{\tau}{\bar{\x}}  &= \bm{I} \\
    \bm{1}^\top \pfrac{\bar{\p}}{\bar{\x}} &= 0
\end{array}
\right.
\end{eqnarray}
Since $\Omega$ is strictly convex, $\bm{H}$ is invertible.
From block Gaussian elimination (\ie, the Schur complement),
\[
    \pfrac{\tau}{\bar{\x}} = -\frac{1}{\bm{1}^\top \bm{H}^{-1} \bm{1}}
    \bm{1}^\top \bm{H}^{-1},
\]
which can then be used to solve for $\pfrac{\bar{\p}}{\bar{\x}}$ giving%
\[
    \pfrac{\bar{\p}}{\bar{\x}} =
    \bm{H}^{-1} -
    \frac{1}{\bm{1}^\top \bm{H}^{-1} \bm{1}} \bm{H}^{-1}\bm{11}^\top
    \bm{H}^{-1},
\]
yielding the desired result. When $\Omega$ is separable, $\bm{H}$ is diagonal,
with $H_{ii} = g''(\pp^\star_i)$, yielding the simplified expression which
completes the proof.
\end{proof}

\paragraph{Connection to other differentiable attention results.}
Our result is similar, but simpler than
\citet[Proposition 1]{fusedmax}, especially in the case of separable $\Omega$.
Crucially, our result does not require that the second derivative exist
outside of the support. As such, unlike the cited work, our result
is applicable in the case of $\aentmax$,
where either $g''(t) = t ^ {\alpha - 2}$ or its reciprocal may not exist
at $t=0$.

\section{Algorithms for \entmaxtext}
\label{sec:allalgo}
\subsection{General thresholded form for bisection algorithms.}
\label{sec:thresholded_form}

The following lemma provides a simplified form for the solution
of $\aentmax$.

\begin{lemma}
\label{lemma:tsallis_reduction}%
For any $\x$, there exists a unique $\tau^\star$ such that
\begin{equation}
\aentmax(\x)
= [(\alpha - 1){\x} - \tau^\star \ones]_+^{\nicefrac{1}{\alpha-1}}.
\end{equation}
\end{lemma}
\begin{proof}
We use the regularized prediction functions defined in \appref{fy_losses}.
From both definitions,
\[\aentmax(\x) \equiv \amap_{-\HHt_{\alpha}}(\x).
\]
We first note that for all $\p \in \triangle^d$,
\begin{equation}
-(\alpha-1) \HHt_\alpha(\p) = \frac{1}{\alpha} \sum_{i=1}^d p_i^\alpha + \text{const}.
\end{equation}
From the constant invariance and
scaling properties of $\amap_{\Omega}$
\citep[Proposition~1, items~4--5]{blondel2019learning},
\[\amap_{-\HHt_\alpha}(\x)
= \amap_{\Omega}((\alpha - 1)\x),
\quad\text{with}\quad
\Omega(\p) =
\sum_{j=1}^{d} g(\pp_j),
\quad
g(t) = \frac{t^\alpha}{\alpha}.
\]
Using \citep[Proposition~5]{blondel2019learning}, noting that
$g'(t) = t^{\alpha - 1}$ and $(g')^{-1}(u) =  u^{\nicefrac{1}{\alpha-1}}$,
yields
\begin{equation}\label{eq:rootform}
\amap_{\Omega}(\x) = [\x - \tau^\star \ones]_+^{\nicefrac{1}{\alpha-1}},
\quad\text{and therefore}\quad
\aentmax(\x) = [(\alpha-1)\x - \tau^\star \ones]_+^{\nicefrac{1}{\alpha-1}}.
\end{equation}
Uniqueness of $\tau^\star$ follows from the fact that $\aentmax$ has a unique
solution $\p^\star$, and \eqnref{rootform}
implies a one-to-one mapping between $\p^\star$ and $\tau^\star$,
as long as $\p^\star \in \simplex$.
\end{proof}

\begin{corollary}\label{lemma:tsallis_15_reduction}
For $\alpha=1.5$, Lemma~\ref{lemma:tsallis_reduction} implies existence of a
unique $\tau^\star$ such that
\[
\aentmax[1.5](\x) = [\nicefrac{\x}{2} - \tau^\star \ones]_+^2.
\]
\end{corollary}

\subsection{An exact algorithm for \entmaxtext with {\boldmath $\alpha=1.5$}:
Derivation of Algorithm~\ref{alg:tsallis_15}.}
\label{sec:algo}

In this section, we derive an exact, sorting-based algorithm for
$\aentmax[1.5]$.  The key observation is that the solution can be characterized
by the size of its support, $\rho^\star = \| \p^\star \|_0$.  Then, we can
simply enumerate all possible values of $\rho \in \{1, \dots, d\}$ until the
solution verifies all optimality conditions.  The challenge,  however, is
expressing the threshold $\thresh$ as a function of the support size $\rho$;
for this, we rely on $\alpha=1.5$.

\begin{proposition}{\label{prop:tsallis_15}Exact computation of $\aentmax[1.5](\x)$}

\noindent Let $\xx_{[d]} \le \dots \le \xx_{[1]}$ denote the sorted coordinates of $\x$,
and, for convenience, let $\xx_{[d+1]} \coloneqq -\infty$.
Define the top-$\rho$ mean, unnormalized variance, and induced threshold for
$\rho \in \{1, \dots, d\}$ as
\begin{equation*}
M_{\x}(\rho)\!\coloneqq\!%
\frac{1}{\rho}\sum_{j=1}^{\rho}\!\xx_{[j]},\quad
S_{\x}(\rho)\!\coloneqq\!\sum_{j=1}^{\rho}\!\left(\xx_{[j]} -
M_{\x}(\rho)\right)^2,
\quad
\tau_{\x}(\rho) \coloneqq
    \begin{cases}
        M_{\x}(\rho) - \sqrt{\frac{1-S_{\x}(\rho)}{\rho}},&S_{\x}(\rho) \leq 1,\\
        +\infty,&S_{\x}(\rho) > 1.
    \end{cases}
\end{equation*}
Then,
\begin{equation}
    \left(\aentmax[1.5](\x)\right)_i = \left[\frac{\xx_i}{2} -
    \tau_{\nicefrac{\x}{2}}(\rho) \right]_+^2,
\end{equation}
for any $\rho$ satisfying $\tau_{\x}(\rho) \in [\xx_{[\rho+1]}, \xx_{[\rho]}]$.
\end{proposition}

Proposition~\ref{prop:tsallis_15} implies the correctness of
Algorithm~\ref{alg:tsallis_15}. To prove it, we first need the following lemma.

\begin{lemma}\label{lemma:tau_nondecr}%
Define $\tau(\rho)$ as in Proposition~\ref{prop:tsallis_15}.
Then, $\tau$ is non-decreasing, and
there exists
$\rho_{\textsf{max}} \in \{1, \dots, d\}$ such that $\tau$ is finite
for $1\leq\rho\leq\rho_\textsf{max}$, and infinite for $\rho>\rho_\textsf{max}$.
\end{lemma}

The proof is slightly more technical, and we defer it to \emph{after} the proof
of the proposition.

\paragraph{Proof of Proposition~\ref{prop:tsallis_15}.}

First, using Corollary~\ref{lemma:tsallis_15_reduction} we reduce the problem
of computing $\aentmax[1.5]$ to
\begin{equation}\label{eq:tsallis_15_reduced_prob}
\amap_\Omega(\x) \coloneqq \argmax_{\p \in \simplex^d}
\p^\top\x - \sum_j \nicefrac{2}{3}~\pp_j^{\nicefrac{3}{2}}.
\end{equation}
Denote by $\tau^\star$ the optimal threshold as defined in the corollary.
We will show that $\tau^\star=\tau(\rho)$ for any $\rho$ satisfying
$\tau(\rho) \in [\xx_{[\rho+1]}, \xx_{[\rho]}]$,
where we assume, for convenience, $\xx_{[d+1]}=-\infty$.
The generic stationarity condition in \eqnref{stationarity}, applied to
the problem in \eqnref{tsallis_15_reduced_prob}, takes the form
\begin{equation}
    \sqrt{\pp_j} = \nu_j + \xx_j - \tau \quad \forall~0 < j \leq d
        \label{eq:stationarity_tsallis_15}\\
\end{equation}
Since $\Omega$ is symmetric, $\amap_{\Omega}$ is permutation-preserving
\citep[Proposition~1, item~1]{blondel2019learning}, so we may
assume w.l.o.g.\ that $\x$ is sorted non-increasingly, i.e.,
$\xx_1 \geq \dots \geq \xx_d$; in other words, $\xx_j = \xx_{[j]}$.
Therefore, the optimal $\p$ is also non-increasing.
Denote by $\rho$ an index such as
$\pp_j \geq 0$ for $1 \leq j \leq \rho$, and $\pp_j = 0$ for $j > \rho$.
From the complementary slackness condition
\eqref{eq:complementary_slack}, $\nu_j=0$ for $1 \leq j \leq \rho$,
thus we may split the stationarity conditions
\eqref{eq:stationarity_tsallis_15} into
\begin{empheq}[left=\empheqlbrace]{align}
    \sqrt{\pp_j} = \xx_j - \tau,  &\quad\forall~1 \leq j \leq \rho, \label{eq:sqrt}\\
    \nu_j = \tau - \xx_j,       &\quad\forall~\rho < j \leq d.
\end{empheq}
For \eqref{eq:sqrt} to have solutions, the r.h.s.\ must be non-negative, i.e.,
$\tau \leq\xx_j$ for $j \leq \rho$, so $\tau \leq \xx_\rho$.
At the same time, from dual feasability \eqref{eq:dual_feas}
we have $\nu_j=\tau - \xx_j \geq 0$ for $j >
\rho$, therefore
\begin{equation}
    \tau(\rho) \in [\xx_{\rho+1}, \xx_\rho].
\label{eq:tau_bound}
\end{equation}
Given $\rho$, we can solve for $\tau$ using \eqref{eq:sqrt} and primal
feasability \eqref{eq:primal_feas}
\begin{equation}
    1 = \sum_{j=1}^d \pp_j = \sum_{j=1}^\rho (\xx_j - \tau)^2.
\end{equation}
Expanding the squares and dividing by $2\rho$ yields the quadratic equation
\begin{equation}\label{eq:quad}
    \frac{1}{2} \tau^2 - \frac{\sum_{j=1}^\rho \xx_j}{\rho} \tau
    + \frac{\sum_{j=1}^\rho \xx_j^2 - 1}{2\rho} = 0,
\end{equation}
with discriminant
\begin{equation}
    \Delta(\rho) = \big(M(\rho)\big)^2 - \frac{\sum_{j=1}^\rho \xx_j^2}{\rho} + \frac{1}{\rho} =
    \frac{1-S(\rho)}{\rho}.
\end{equation}
where we used the variance expression
$\EE\left[\left(X-\EE[X]\right)^2\right]= \EE[X^2] - \EE[X]^2$.
If $S(\rho)>1$, $\Delta(\rho)<0$, so there must exist an optimal $\rho$ satisfying
$S(\rho) \in [0, 1]$. Therefore, \eqref{eq:quad} has the two solutions
$\tau_{\pm}(\rho) = M(\rho) \pm \sqrt{\frac{1-S(\rho)}{\rho}}$.
However, $\tau_+$ leads to a contradiction: The mean $M(\rho)$ is never smaller than the
smallest averaged term, so $M(\rho) \geq \xx_\rho$, and thus $\tau_+ \geq \xx_\rho$.
At the same time, from \eqref{eq:tau_bound}, $\tau\leq \xx_\rho$, so $\tau$ must equal
$\xx_\rho$, which can only happen if $M(\rho)=\xx_\rho$ and $S(\rho)=1$.
But $M(\rho) = \xx_\rho$ only if $\xx_1 = \dots = \xx_\rho$, in which case
$S(\rho)=0$ (contradiction).

Therefore, $\tau^\star = \tau(\rho) = M(\rho) - \sqrt{\frac{1-S(\rho)}{\rho}}$ for
\emph{some} $\rho$ verifying \eqref{eq:tau_bound}. It remains to show that
\emph{any} such $\rho$ leads to the same value of $\tau(\rho)$.
Pick any $\rho_1 < \rho_2$, both verifying \eqref{eq:tau_bound}.
Therefore, $\rho_1 + 1 \leq \rho_2$ and
\begin{equation}
\xx_{\rho_1 + 1} \underbrace{\leq}_{\eqref{eq:tau_bound}\text{ for }\rho_1} \tau(\rho_1)
\underbrace{\leq}_{\text{Lemma }\ref{lemma:tau_nondecr}} \tau(\rho_2)
\underbrace{\leq}_{\eqref{eq:tau_bound}\text{ for }\rho_2} \xx_{\rho_2}
\underbrace{\leq}_{\x \text{ sorted}} \xx_{\rho_1+1},
\end{equation}
thus $\tau(\rho_1) = \tau(\rho_2)$, and so any $\rho$ verifying \eqref{eq:tau_bound}
satisfies $\tau^\star=\tau(\rho)$,
concluding the proof.
\hfill\qedsymbol

\paragraph{Proof of Lemma~\ref{lemma:tau_nondecr}.}
We regard $\tau(\rho)$ as an extended-value sequence, \ie, a function from $\mathbb{N} \rightarrow \reals \cup {\infty}$.
The lemma makes a claim about the \emph{domain} of the sequence $\tau$, and a
claim about its monotonicity. We prove the two in turn.

\noindent\textbf{Domain of $\tau$.}
The threshold $\tau(\rho)$ is only finite for $\rho \in T \coloneqq \big\{\rho \in
\{1, \dots, d\}  \colon S(\rho) \leq 1\big\}$, i.e., where $\nicefrac{(1-S(\rho))}{\rho} \geq 0$.
We show there exists $\rho_\textsf{max}$ such that $T=\{1, \dots, \rho_\textsf{max}\}$.
Choose $\rho_\textsf{max}$ as the largest index satisfying
$S(\rho_\textsf{max})\leq1$.
By definition, $\rho>\rho_\textsf{max}$ implies $\rho \notin T$.
Remark that $S(1)=0$, and $S(\rho+1) - S(\rho) = (\cdot)^2 \geq 0$.
Therefore, $S$ is nondecreasing and, for any $1 \leq \rho \leq \rho_\textsf{max},
0 \leq S(\rho) \leq 1$.

\noindent\textbf{Monotonicity of $\tau$.}
Fix $\rho \in [\rho_\textsf{max}-1]$,
assume w.l.o.g.\ that $M_{\x}(\rho) = 0$, and define $\tilde\x$ as
\[
    \tilde\xx_{[j]} = \begin{cases}
        x, & j=\rho+1, \\
        \xx_{[j]}, & \text{otherwise}. \\
    \end{cases}
\]
The $\rho$ highest entries of $\tilde\x$ are the same as in $\x$, so
$M_{\tilde\x}(\rho) = M_{\x}(\rho) = 0$,
$S_{\tilde\x}(\rho) = S_{\x}(\rho)$, and
$\tau_{\tilde\x}(\rho) = \tau_{\x}(\rho)$.
Denote $\widetilde\tau(x) \coloneqq \tau_{\tilde\x}(\rho+1)$,
and analogously $\widetilde M(x)$ and $\widetilde S(x)$.
Then,
\begin{equation}\label{eq:tau_chain}
\tau_{\x}(\rho+1) = \widetilde\tau(\xx_{[\rho+1]}) \geq
\min_{x\colon \widetilde S(x) \in [0,1]} \widetilde \tau(x)
\eqqcolon \widetilde\tau(x^\star)
\end{equation}
We seek the lower bound
$\widetilde\tau(x^\star)$
and show that $\widetilde\tau(x^\star) \geq \tau_{\x}(\rho)$.
From \eqref{eq:tau_chain}, this implies
$\tau_{\x}(\rho+1) \geq \tau_{\x}(\rho)$ and,
by transitivity, the monotonicity of $\tau_{\x}$.

It is easy to verify that the following incremental update expressions hold.
\begin{equation}
    \widetilde M(x) = \frac{x}{\rho+1}, \quad\quad\quad\quad
    \widetilde S(x) = S_{\x}(\rho) + \frac{\rho}{\rho+1} x^2.
\end{equation}
We must solve the optimization problem
\begin{equation}\label{eq:mintau}
    \text{minimize}_{x}\quad \widetilde\tau(x) \quad \text{subject
    to}\quad\widetilde S(x) \in [0, 1].
\end{equation}
The objective value is
\begin{equation}
    \widetilde\tau(x) = \widetilde M(x) - \sqrt{\frac{1-\widetilde S(x)}{\rho+1}} =
    \frac{1}{\rho+1} \left(x - \sqrt{\big(1-S_{\x}(\rho)\big)(\rho+1) - \rho x^2} \right)
\end{equation}
Ignoring the constraint for a moment and setting the gradient to $0$ yields the solution
\begin{equation}
\begin{aligned}
    0 = \widetilde{\tau}'(x^\star) &= \frac{1}{\rho+1} \left(1 +\frac{\rho
    x^\star}{\sqrt{\big(1-S_{\x}(\rho)\big)(\rho+1) - \rho {x^\star}^2}} \right) \\
\iff\quad
            \rho x^\star &= -\sqrt{\big(1-S_{\x}(\rho)\big)(\rho+1) - \rho {x^\star}^2},
\end{aligned}
\end{equation}
implying $x^\star < 0$. Squaring both sides and rearranging yields
the solution of the unconstrained problem,
\begin{equation}
    x^\star = - \sqrt{\frac{1-S_{\x}(\rho)}{\rho}}.
\end{equation}
We verify that $x^\star$ readily satisfies the constraints,
thus it is a solution to the minimization in \eqnref{mintau}:
\begin{equation}
    0 \leq S_{\x}(\rho) \leq \widetilde S(x^\star) =
    S_{\x}(\rho) + \frac{1-S_{\x}(\rho)}{\rho+1}
    \leq S_{\x}(\rho) + \frac{1-S_{\x}(\rho)}{2} \leq 1.
\end{equation}
Plugging $x^\star$ into the objective yields
\begin{equation}
    \begin{aligned}
        \widetilde\tau(x^\star) &= \frac{1}{\rho+1}
        \left(-\sqrt{\frac{1-S_{\x}(\rho)}{\rho}} -
        \sqrt{\rho\big(1-S_{\x}(\rho)\big)} \right)  \\
        &= -\sqrt{\frac{1-S_{\x}(\rho)}{\rho}} \frac{1}{\rho+1} (1 + \rho) \\
        &= -\sqrt{\frac{1-S_{\x}(\rho)}{\rho}} = \tau_{\x}(\rho).
    \end{aligned}
\end{equation}
Therefore, $\widetilde\tau(x) \geq \tau_{\x}(\rho)$ for any valid $x$,
proving that
$\tau_{\x}(\rho) \leq \tau_{\x}(\rho+1)$.
\end{document}